\documentclass[conference]{IEEEtran}
\usepackage[table]{xcolor}
\usepackage{eda} 
\usepackage{microtype}
\usepackage{graphicx}
\usepackage{booktabs} 
\usepackage{amssymb}
\usepackage{amsmath}
\usepackage{amsthm}
\usepackage{enumitem}
\usepackage{physics}
\usepackage{rotating}
\usepackage{mathtools}
\usepackage{algorithm}
\usepackage[noend]{algpseudocode}
\usepackage{caption}
\usepackage{subcaption}
\usepackage[nocompress]{cite}
\usepackage{balance}
\usetikzlibrary{positioning,calc}

\usepackage{tikz}
\usetikzlibrary{snakes,arrows,shapes,positioning,fit,calc,patterns}
\usepackage{pgfplots}
\tikzset{cross/.style={cross out, draw=black, fill=none, minimum size=2*(#1-\pgflinewidth), inner sep=0pt, outer sep=0pt}, cross/.default={2pt}}
\usetikzlibrary{external}
\usepackage{nicefrac}

\newtheorem{theorem}{Theorem}
\newtheorem{lemma}[theorem]{Lemma}
\newtheorem{proposition}[theorem]{Proposition}

\newtheorem{definition}{Definition}

\DeclarePairedDelimiter{\ceil}{\lceil}{\rceil}
\DeclareMathOperator*{\argmax}{arg\,max}
\newcommand{\defemph}[1]{\textbf{#1}}
\newcommand{\mapping}[3]{#1 \with #2 \to #3}
\newcommand{\card}[1]{| #1 |}
\newcommand{\with}{\! : \,}
\newcommand{\given}{\, | \,}
\newcommand{\setupto}[1]{[#1]}

\newcommand{\hatp}{\ensuremath{\hat{p}}}

\newcommand{\hatpg}[4]{\ensuremath{\hat{p}_{\scriptscriptstyle {#4} | {#1}\text{=}{#2}}({#3})}}

\newcommand{\gen}[2]{\ensuremath{#1 \preceq #2}}
\newcommand{\homo}[2]{\ensuremath{#1 \precsim #2}}

\newcommand{\cA}{\mathcal{A}}
\newcommand{\cC}{\mathcal{C}}

\newcommand{\cI}{\mathcal{I}}

\newcommand{\cS}{\mathcal{S}}
\newcommand{\cT}{\mathcal{T}}

\newcommand{\cX}{\mathcal{X}}

\newcommand{\cZ}{\mathcal{Z}}
\newcommand{\cXp}{\cX^{+}}
\newcommand{\cXo}{\mathcal{X}^{*}}
\newcommand{\bd}{\mathbf{d}}
\newcommand{\cB}{\mathcal B}

\newcommand{\oest}{\bar{f}}

\newcommand{\rfn}{\ensuremath{\text{spc}}}
\newcommand{\crd}{\ensuremath{\text{mon}}}
\newcommand{\methodA}{\textbf{\textsc{\OPUS}}_{\rfn}\xspace}
\newcommand{\methodB}{\textbf{\textsc{\OPUS}}_{\crd}\xspace}
\newcommand{\methodC}{\textbf{\textsc{\BNB}}_{\crd}\xspace}
\newcommand{\methodD}{\textbf{\textsc{\BEAM}}_{\rfn}\xspace}
\newcommand{\methodE}{\textbf{\textsc{\BEAM}}\xspace}

\newcommand{\OPUS}{\textsc{Opus}}
\newcommand{\Opus}{\textsc{Opus}\xspace}

\newcommand{\BNB}{\textsc{BnB}}
\newcommand{\BEAM}{\textsc{Greedy}}

\newcommand{\fbarn}{\ensuremath{\bar{f}_{\text{spc}}}\xspace}
\newcommand{\fbaro}{\ensuremath{\bar{f}_{\text{mon}}}\xspace}

\newcommand{\bo}{\hat{b}_0\xspace}

\newcommand{\mo}{\hat{m}_o\xspace}

\newcommand{\X}{\mathcal{X}}

\newcommand{\vx}{\mathbf{x}}
\newcommand{\vc}{\mathbf{c}}

\newcommand{\NMI}{F\xspace}
\newcommand{\hNMI}{\ensuremath{\hat{F}}\xspace}
\newcommand{\hI}{\ensuremath{\hat{I}}\xspace}
\newcommand{\hH}{\ensuremath{\hat{H}}\xspace}

\newcommand{\ourScore}{\ensuremath{\hat{F}_{0}}}
\newcommand{\Io}{\hat{I}_0\xspace}

\newcommand{\D}{\mathbf{D}}

\newcommand{\domain}[1]{V(#1)}

\graphicspath{ {./} }

\newcommand{\giturl}{\url{https://github.com/pmandros/fodiscovery}}

\ifpdf
\usetikzlibrary{external}
\fi

\def\BibTeX{{\rm B\kern-.05em{\sc i\kern-.025em b}\kern-.08em
		T\kern-.1667em\lower.7ex\hbox{E}\kern-.125emX}}
\begin{document}
	
	\title{Discovering Reliable Dependencies from Data: Hardness and Improved Algorithms}
	
\author{\IEEEauthorblockN{Panagiotis Mandros, Mario Boley, Jilles Vreeken}
	\IEEEauthorblockA{Max Planck Institute for Informatics and Saarland University\\
		Saarland Informatics Campus, Saarbr\"{u}cken, Germany\\
		\{pmandros,mboley,jilles\}@mpi-inf.mpg.de	
	}
}
	
	\maketitle
	
	\begin{abstract}
	The reliable fraction of information is an attractive score for quantifying (functional) dependencies in high-dimensional data. In this paper, we systematically explore the algorithmic implications of using this measure for optimization. We show that the problem is NP-hard, which justifies the usage of worst-case exponential-time as well as heuristic search methods. We then substantially improve the practical performance for both optimization styles by deriving a novel admissible bounding function that has an unbounded potential for additional pruning over the previously proposed one. Finally, we empirically investigate the approximation ratio of the greedy algorithm and show that it produces highly competitive results in a fraction of time needed for complete branch-and-bound style search. 
	\end{abstract}
	

	\section{Introduction}
Given a data sample $\D_n=\{\bd_1,\dots,\bd_n\}$ drawn from the joint distribution $p$ of some input variables $\cI$ and an output variable $Y$,
it is a fundamental problem in data analysis to find variable subsets $\cX \subseteq \cI$ that jointly influence or (approximately) determine $Y$.
This \defemph{functional dependency discovery} problem, i.e., to find
\begin{equation}\label{eq:problem}
\argmax \{ Q(\cX;Y) : \cX \subseteq \cI\}
\end{equation}
for some real-valued measure $Q$ that assesses the dependence of $Y$ on $\cX$,
is a classic topic in the database community~\cite[Ch. 15]{ramakrishnan:2000:dbms}, but also has many other applications including feature selection \cite{song:2012:bahsic} and knowledge discovery~\cite{ziarko:2002:rough}. For instance, finding such dependencies can help identify compact sets of descriptors that capture the underlying structure and actuating mechanisms of complex scientific domains (e.g., \cite{ghiringhelli:2015:big,ouyang:2017:sisso}).

For categoric input and output variables, the measure $Q$ can be chosen to be the \defemph{fraction of information}~\cite{cavallo:1987:fraction, giannella:2004:fraction, reimherr:2013:fraction} defined as
\[
F(\cX;Y)=(H(Y)-H(Y\given \cX))/H(Y) \enspace,
\]
where $H(Y)=\sum_{y \in Y}p(y)\log p(y)$ denotes the \defemph{Shannon entropy}.
This score represents the relative reduction of uncertainty about $Y$ given $\cX$. It takes on values between $0$ and $1$ corresponding to independence and exact functional dependency, respectively. 

Estimating the score naively with empirical probabilities $\hatp$, however, leads to an overestimation of the actual dependence between $\cX$ and $Y$, a behavior known as \emph{dependency-by-chance}~\cite{romano:2016:chance}. In particular, since the bias is increasing with the domain size of variables \cite{roulston1999estimating}, it is unsuitable for dependence discovery where we have to soundly compare different variable sets of varying dimensionality and consequently of widely varying domain sizes (see Fig.~\ref{fig:cardVsInd}). 
In some feature selection approaches (see, e.g., \cite{guyon:2003:featsel}) this problem is mitigated by only considering dependencies of individual variables or pairs.
Alternatively, some algorithms from the database literature, e.g., \cite{huhtala:1999:tane,kruse:2018:afd}, neglect this issue by assuming a closed-world, i.e., the unknown data generation process $p$ is considered equal to the empirical $\hatp$~\cite{giannella:2004:fraction}.

\begin{figure}[t]

	\includegraphics[width=\linewidth]{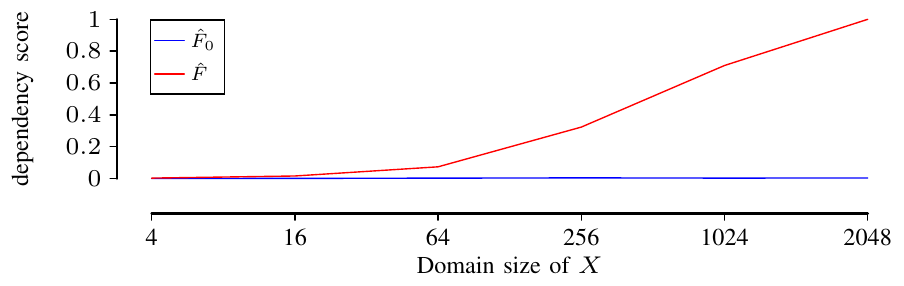}
	\caption{\textbf{Dependency-by-chance.} 
	Estimated fraction of information for variables $X$ of increasing domain size (4 to 2048) to independent $Y$ (domain size 4) for fixed sample size (1000).
	Estimated dependency increases for naive estimator $\hNMI$, while the corrected-for-chance estimator $\ourScore$ accurately estimates population value $\NMI(X;Y)=0$.}
	\label{fig:cardVsInd}
\end{figure}

Both of these approaches are infeasible in the statistical setting with arbitrary sized variable sets that we are interested in. Instead, here, the fraction of information can be corrected by subtracting its estimated expected value under the hypothesis of independence. This gives rise to the \defemph{reliable fraction of information}~\cite{mandros:2017:discovering,nguyen:2010:chancejournal} defined as
\begin{equation}
\label{eq:relfoi}
\ourScore(\cX;Y)= \hNMI(\cX;Y)-\hat{E}_0(\hNMI(\cX;Y)) \enspace,
\end{equation}
where $\hat{E}_0(\hNMI(\cX;Y))=\nicefrac{ \sum_{\sigma \in S_n} \hNMI(X;Y_\sigma)}{n!}$ is the expected value of $\hNMI$ under the \defemph{permutation model}~\cite[p. 214]{lancaster:1969:chi}, i.e., under the operation of permuting the empirical $Y$ values with a random permutation $\sigma \in S_n$.
This estimator can be computed efficiently in time $O(nk)$ for $\cX$ with domain size $k$ (see \cite{romano:2014:smi} and appendix). Moreover, the  maximization problem (Eq.~\eqref{eq:problem}) can be solved effectively by a simple branch-and-bound scheme: the maximally attainable $\ourScore$ for  supersets of some partial solution $\cX$ can be bounded by the function $\fbaro(\cX)=1-\hat{E}_0(\hNMI(\cX;Y))$, which follows from the \emph{monotonicity} of $\hat{E}_0(\hNMI(\,\cdot\,;Y))$~\cite{mandros:2017:discovering}.

This, however, is a rather simplistic bounding function that leaves room for substantial improvements. Moreover, it is unclear whether one has to rely on exponential-time worst-case branch-and-bound algorithms in the first place. Finally, the option of heuristic optimization has not yet been explored. 

To this end, this paper provides the following contributions:
\begin{enumerate}
\item We show that the problem of maximizing the reliable fraction of information is NP-hard. This justifies the usage of worst-case exponential-time algorithms as well as heuristic search methods (Sec.~\ref{sec:nphard}).
\item Motivated by this insight, we then greatly improve the practical performance for both of these optimization styles by deriving a novel admissible bounding function $\fbarn(\cX)$. This function is not only tighter than the previously proposed $\fbaro(\cX)$ but in particular we have that the supremum of $\fbaro(\cX)/\fbarn(\cX)$---and thus the potential for additional pruning in search---is unbounded (Sec.~\ref{sec:oest}).
\item Finally, we report extensive empirical results evaluating the proposed bounding function and the various algorithmic strategies.
In particular, we consider the approximation ratio of the greedy algorithm and show that in fact, it produces highly competitive results in a fraction of time needed for complete branch-and-bound style search---motivating further investigation of this fact (Sec.~\ref{sec:eval}). 
\end{enumerate}
We round up with a concluding discussion (Sec.~\ref{sec:conc}). Before presenting the main contributions, we recall reliable functional dependency discovery and prove some basic results (Sec.~\ref{sec:prelim}).
	\section{Reliable Dependency Discovery}
\label{sec:prelim}
Let us denote by $\setupto{n}$ the set of positive integers up to $n$. The symbols $\log$ and $\ln$ refer to the logarithms of base $2$ and $e$, respectively. We assume a set of discrete random variables $\cA=\cI \cup \{Y\}$ is given along with an empirical sample $\D_n=\{\bd_1,\dots,\bd_n\}$ of their joint distribution.
For a variable $X$ we denote its domain, called \defemph{categories} (or distinct values), by $\domain{X}$ but we also write $x \in X$ instead of $x \in \domain{X}$ whenever clear from the context.
We identify a random variable $X$ with the \defemph{labeling} $\mapping{X}{\setupto{n}}{\domain{X}}$ it induces on the data sample, i.e., $X(i)=\bd_i(X)$. 
Moreover, for a set $\cS=\{S_1,\dots,S_l\}$ of labelings over $\setupto{n}$, we define the corresponding vector-valued labeling by
$
\cS(i)=(S_{1}(i), \dots, S_{l}(i)) 
$.
With $X_{\mathcal{Q}}$ for a subset $\mathcal{Q} \subseteq \setupto{n}$, we denote the map $X$ restricted to domain $\mathcal{Q}$.

We define $\mapping{c_{\scriptscriptstyle X}}{\domain{X}}{\mathbb{Z}_+}$ to be the \defemph{empirical counts} of $X$, i.e., $c_{\scriptscriptstyle X}(x)=\card{\{i \in \setupto{n} \with X(i)=x \}}$. We further denote with $\hat{p}_{\scriptscriptstyle X} \with \domain{X} \rightarrow [0,1]$, where $\hat{p}_{\scriptscriptstyle X}(x)=\nicefrac{c_{\scriptscriptstyle X}(x)}{n}$, the \defemph{empirical distribution} of $X$. Given another random variable $Z$, $\hat{p}_{\scriptscriptstyle Z \given {X}\text{=}{x}} \with \domain{Z} \rightarrow [0,1]$ is the \defemph{empirical conditional distribution} of $Z$ given $X=x$, with $\hatpg{X}{x}{z}{Z}=\nicefrac{c_{ \scriptscriptstyle X \cup Z}(x, z)}{c_{\scriptscriptstyle X}(x)}$ for $z \in Z$. However,  we use $\hatp(x)$ and $\hatp(z \given x)$ respectively whenever clear from the context. These empirical probabilities give rise to the \defemph{empirical conditional entropy} $\hH(Y \given X)=\sum_{x \in X}\hat{p}(x)\hH(Y \given X=x)$, the \defemph{empirical mutual information} $\hI(X;Y)=\hH(Y)-\hH(Y \given X)$, and the \defemph{empirical fraction of information} $\hNMI(X;Y)=\hI(X;Y)/\hH(Y)$. 

Recall that the reliable fraction of information is defined as the empirical fraction of information $\hNMI(X;Y)$ minus its expected value under the permutation model $\hat{E}_0(\hNMI(X;Y))$  
where $\hat{E}_0(\hNMI(X;Y))=\sum_{\sigma \in S_n}\hNMI(X;Y_\sigma)/n!$. Here, $S_n$ denotes the \defemph{symmetric group} of $\setupto{n}$, i.e., the set of bijections from $\setupto{n}$ to $\setupto{n}$, and $A_\sigma$ denotes the composition of a map $A$ with the permutation $\sigma \in S_n$, i.e.,  $A_\sigma(\cdot)=A(\sigma(\cdot))$. We abbreviate the \defemph{correction term} $\hat{E}_0(\hNMI(X;Y))$ as $\hat{b}_0(X,Y,n)$ and the unnormalized version as $\hat{m}_0(X,Y,n)=\hat{b}_0(X,Y,n)\hH(Y)$.

\subsection{Specializations and Labeling Homomorphisms}
Since we identified sets of random variables with their corresponding sample-index-to-value map, they are subject to the following general relations of maps with common domains.
\newcommand{\labeq}{\equiv}
\begin{definition}
Let $A$ and $B$ be maps defined on a common domain $N$. We say that $A$ is \defemph{equivalent} to $B$, denoted as $A \equiv B$, if for all $i,j \in N$ it holds that $A(i)=A(j)$ if and only if $B(i) = B(j)$. 
We say that $B$ is a \defemph{specialization} of $A$, denoted as $A \preceq B$, if for all $i,j \in N$ with $A(i) \neq A(j)$ it holds that $B(i) \neq B(j)$. 
\end{definition}
\noindent  A special case of specializations is given by the subset relation of variable sets, e.g., if $\cX \subseteq \cX' \subseteq \cI$ then $\gen{\cX}{\cX'}$. 
The specialization relation implies some important properties for empirical probabilities and information-theoretic quantities.
 \begin{proposition}\label{prop:spec}
 	Given variables $X$, $Z$ and $Y$, with $X \preceq Z$, the following statements hold:
 	\begin{enumerate}[label=\alph*)]
 		\item there is a projection $\mapping{\pi}{\domain{Z}}{\domain{X}}$, s.t. for all $x \in \domain{X}$, it holds that $\hatp_{\scriptscriptstyle X}(x)=\sum_{z \in \pi^{-1}(x)} \hatp_{\scriptscriptstyle Z}(z)$,
 		\item $\hH(X) \leq \hH(Z)$
 		\item $\hH(Y \given Z) \leq \hH(Y \given X) $,
 		\item\label{spec:mutinf} $\hI(X;Y) \leq \hI(Z;Y)$,
 	\end{enumerate}
 \end{proposition}
 \begin{proof} 
 	Let us denote with $p$ and $q$ the $\hatp_{\scriptscriptstyle X \cup Y}$ and $\hatp_{\scriptscriptstyle Z \cup Y}$ distributions respectively.
 	Statement a) follows from the definition. For b), we define $h(x) = -p(x) \log p(x) $ for $x \in X$, and similarly $h(z)$ for $z \in Z$. We show that for all $x \in X$,  $h(x) \leq \sum_{z \in \pi^{-1}(x)}h(z)$.  The statement then follows from the definition of $\hH$. We have
 	 \begin{align*}
 	h(x) &= -p(x) \log p(x) \\
 	 	&=-\left ( \sum_{z \in \pi^{-1}(x)}q(z) \right ) \log \left ( \sum_{z \in \pi^{-1}(x)}q(z) \right ) \\
 	&=-\sum_{z \in \pi^{-1}(x)} \left ( q(z)  \log \left(\sum_{s \in \pi^{-1}(x)}q(s) \right)\right)\\
 	&\leq -\sum_{z \in \pi^{-1}(x)} q(z)  \log q(z) = \sum_{z \in \pi^{-1}(x)}h(z) \enspace,
 	\end{align*}
 	where the inequality follows from the monotonicity of the $\log$ function (and the fact that $q(z)$ is positive for all $z \in Z$).\\
 	c) Let us first recall the log-sum inequality~\cite[p.~31]{cover}: for non-negative numbers $a_1,a_2, \dots, a_n$ and $b_1, b_2, \dots, b_n$, 
 	\begin{equation}\label{eq:logsum}
 	\sum_{i=1}^{n}a_i \log\frac{a_i}{b_i} \geq  	\Big(\sum_{i=1}^{n}a_i \Big) \frac{	\sum_{i=1}^{n}a_i}{	\sum_{i=1}^{n}b_i}
 	\end{equation}
 	with equality if and only if $a_i/b_i$ constant. We have:
 	\begin{align*}
 	\hH(Y \given Z) = & -\sum_{z \in Z, y \in Y} q(z,y) \log \frac{q(z,y)}{q(z)} \\
   	\stackrel{(a)}{=}& -\sum_{x \in X, y \in Y} \sum_{z \in \pi^{-1}(x)}  q(z,y) \log \frac{q(z,y)}{q(z)} \\
 	\stackrel{\eqref{eq:logsum}}{\leq} & -\sum_{x \in X, y \in Y} \Big( \sum_{z \in \pi^{-1}(x)}  q(z,y)  \Big) \frac{\displaystyle \sum_{z \in \pi^{-1}(x)} q(z,y) }{\displaystyle \sum_{z \in \pi^{-1}(x)}  q( z) }  \\
 	= & -\sum_{x \in X, y \in Y} p(x,y) \log \frac{p(x,y)}{p(x)} = \hH(Y \given X)
 	\end{align*}
 	d) We have $\hI(Z;Y)=\hH(Y)-\hH(Y \given Z)$ $ \leq$ $ \hH(Y)-\hH(Y \given X)=\hI(X;Y)$ following from (c). 
 \end{proof}

In order to analyze monotonicity properties of the permutation model, the following additional definition is useful.
\begin{definition}
We call a labeling $X$ \defemph{homomorphic} to a labeling $Z$ (w.r.t. the target variable $Y$), denoted as $\homo{X}{Z}$, if there exists $\sigma \in S_n$ with $Y \equiv Y_\sigma$ such that $X \preceq Z_\sigma$.
\end{definition}
See Tab.~\ref{tab:homoiso} for examples of both introduced relations.
Importantly, the inequality of mutual information for specializations (Prop.~\ref{prop:spec}\ref{spec:mutinf} carries over to homomorphic variables and in turn to their correction terms. 

\begin{proposition}\label{prop:isospec}
	Given variables $X$, $Z$ and $Y$, with $\homo{X}{Z}$, the following statements hold:
	\begin{enumerate}[label=\alph*)]
		\item  $\hI(X;Y) \leq \hI(Z;Y)$ \label{prop:isospec1}
		\item $\mo(X,Y,n) \leq \mo(Z,Y,n)$ \label{isospec:correction}
	\end{enumerate}
\end{proposition}
\begin{proof}
	Let $\sigma^* \in S_n$ be a permutation for which $Y \labeq Y_{\sigma^*}$ and  $X \preceq Z_{\sigma^*} $. Property a) follows from
	\begin{equation*}
	\hI(Z;Y)= \hI(Z_{\sigma^*};Y_{\sigma^*}) =  \hI(Z_{\sigma^*};Y)\geq \hI(X;Y) \enspace,
	\end{equation*}
	where the inequality holds from Prop.~\ref{prop:spec}\ref{spec:mutinf}.
	For b), note that for every $\sigma \in S_n$, it holds from Prop.~\ref{prop:spec}\ref{spec:mutinf} that $\hI(Z_{ \sigma \circ \sigma^*};Y) \geq \hI(X_{\sigma};Y)$. Hence
	\begin{align*}
\mo(Z,Y,n)=&\frac{1}{n!} \sum_{\sigma \in S_n} \hI(Z_{\sigma};Y)  \\
=& \frac{1}{n!} \sum_{\sigma \in S_n} \hI(Z_{\sigma \circ \sigma^*};Y) \\ 
\geq& \frac{1}{n!}\sum_{\sigma \in S_n} \hI(X_\sigma;Y) =\mo(X,Y,n)
	\end{align*}
\end{proof}

\begin{table}[]
	\centering
	\begin{tabular}{ccccc}
		\toprule
		$X_1$ & $X_2$ & $X_3$ & $X_4$ & $Y$ \\  \midrule
		a   & a     & a     & b     & a   \\
		a     & b    & b    & a     & b   \\
		b     & c     & b     & b     & b   \\
		b     & c    & c     & c     & b  \\ \bottomrule
	\end{tabular}
	\caption{\textbf{Specialization and homomorphism examples}. We have $X_1 \preceq X_2$, $\homo{X_1}{X_2}$, $\homo{X_1}{X_3}$, $\homo{X_1}{X_4}$, $\homo{X_2}{X_3}$. Note that  $ X_3 \not \precsim X_4$ as there is no $\sigma \in S_4$ that satisfies specialization w.r.t. $X_{4}$~and~$Y \equiv Y_\sigma$}
	\label{tab:homoiso}
\end{table}

\begin{algorithm}[th]
	\caption{
		$\OPUS$: Given a set of input variables $\cI$, function $f$, bounding function $\oest$, and $\alpha \in (0,1]$,  the algorithm returns the $\cX^* \subseteq \cI$ satisfying $f(\cX^*) \geq \alpha \max \{f(\cX') \with \cX' \subseteq \cI \}$
	}
	\label{alg:opus}
	\begin{algorithmic}[1]
		\Function{$\OPUS$}{$\mathbf{Q},\cS$} 
		\If{$\mathbf{Q}$ is empty}
		\State \textbf{return} $\cS$
		\Else
		\State $(\cX, \cZ)=pop(\mathbf{Q})$
		\State $\mathbf{R}=\{ (\cX \cup \{Z\},Z) \with Z \in \cZ \}$
		\State $\cX^* = \argmax    \{f(\cX') \with \cX' \in \mathbf{R} \cup \{\cS\} \} $
		\State $\mathbf{R}'=\{ (\cX',Z) \in \mathbf{R}  \with \alpha \oest(\cX') > f(\cX^*)   \}  $
		\State $\cZ'=\{Z \with (\cX',Z) \in \mathbf{R}'\}$
		\State $[(\cX_1,Z_1), \dots, (\cX_k, Z_k)]=sort(\mathbf{R}')$
		\State $\mathbf{Q}'=\mathbf{Q} \cup \{ (\cX_i, \cZ' \setminus \{Z_1, \dots, Z_i\} ) \with i \in \setupto{k}     )  \} $
		\State \textbf{return} $\OPUS(\mathbf{Q}',\cX^*) $
		\EndIf
		\EndFunction
		\State $\cX^*=\OPUS(\{  (\emptyset, \cI)  \},\emptyset) $
	\end{algorithmic}
\end{algorithm}
\subsection{Search Algorithms}
Effective algorithms for maximizing the reliable fraction of information over all subsets $\cX \subseteq \cI$ are enabled by the concept of bounding functions.
A function $\oest$ is called an \defemph{admissible bounding function} for an optimization function $f$ if for all candidate solutions $\cX \subseteq \cI$, it holds that $\oest (\cX) \geq f(\cX')$ for all $\cX'$ with $\cX \subseteq \cX' \subseteq \cI$. 
Such functions allow to prune all supersets $\cX'$ of $\cX$ whenever $\oest(\cX) \leq f(\cX^*)$ for the current best solution $\cX^*$ found during the optimization process. 

\defemph{Branch-and-bound}, as the name suggests, 
combines this concept with a branching scheme that completely (and non-redundantly) enumerates the search space $2^\cI$.
Here, we consider \defemph{optimized pruning for unordered search} (\defemph{\Opus}), an advanced variant of branch-and-bound that effectively propagates pruning information to siblings in the search tree~\cite{webb:1995:opus}. 
Algorithm~\ref{alg:opus} shows the details of this approach.

In addition to keeping track of the best solution $\cX^*$ seen so far, the algorithm maintains a priority queue $\mathbf{Q}$ of pairs $(\cX,\cZ)$, where $\cX \subseteq \cI$ is a candidate solution and $\cZ \subseteq \cI$ constitutes the variables that can still be used to augment $\cX$, e.g., $\cX'=\cX \cup \{Z\}$ for a $Z \in \cZ$. The top element is the one with the smallest cardinality and the highest potential (a combination of breadth-first and best-first order). 
Starting with $\mathbf{Q}=\{(\emptyset, \cI)\}$, $\cX^*=\emptyset$, and a desired approximation guarantee $\alpha \in (0,1]$, in every iteration \Opus  creates all refinements of the top element of $\mathbf{Q}$ and updates $\cX^*$ accordingly (lines 5-7). Next the refinements are pruned using $\oest$ and $\alpha$ (line 8).  Following, the pruned list is sorted according to decreasing potential (this is a heuristic that propagates the most augmentation elements to the least promising refinements~\cite{webb:1995:opus}), the possible augmentation elements $\mathcal{Z}'$ are non-redundantly propagated to the refinements of the top element, and finally the priority queue is updated with the new candidates (lines 9-11). 

A commonly used alternative to complete branch-and-bound search for the optimization of dependency measures is the standard \defemph{greedy algorithm} (see \cite{guyon:2003:featsel, brown:2012:featsel}).
This algorithm only refines the best candidate in a given iteration. 
Moreover, bounding functions can be incorporated as an early termination criterion.
For the reliable fraction of information in particular, there is potential to prune many of the higher levels of the search space as it favors solutions that are small in cardinality~\cite{mandros:2017:discovering}. 
The algorithm is presented in Algorithm~\ref{alg:greedy}.


The algorithm keeps track of the best solution $\cX^*$ seen, as well as the best candidate for refinement $\cC^*$. Starting with $\cX^*=\emptyset$ and $\cC^*=\emptyset$, the algorithm in each iteration checks whether $\cC^*$ can be refined further, i.e., if $\cI \setminus \cC^*$ is not empty, or if $\cC^*$ has potential (the early termination criterion). If not, the algorithm terminates returning $\cX^*$ (lines 2-3). Otherwise $\cC^*$ is refined to all possible refinements, and the best one is selected as a candidate to update $\cX^*$ (lines 5-7).

Concerning the approximation ratio of the greedy algorithm, there exists a large amount of research focused on submodular and/or monotone functions, e.g.,~\cite{feige:2011:maximizing,das:2011:subratio,bian:2017:approximate}. However, we note that $\ourScore$ is neither submodular nor monotone, and hence these results are not directly applicable. To demonstrate empirically the quality of the results, we perform an evaluation in Sec.~\ref{sec:greedyeval}. We discuss further on this topic in Sec.~\ref{sec:conc}. 

\begin{algorithm}[t]
	\caption{
		$\BEAM$: Given a set of input variables $\cI$, function $f$, and bounding function $\oest$,  the algorithm returns the $\cX^* \subseteq \cI$ approximating $f(\cX^*)= \max \{f(\cX') \with \cX' \subseteq \cI \}$
	}
	\label{alg:greedy}
	\begin{algorithmic}[1]
		\Function{$\BEAM$}{$\cC,\cS$} 
		\If{$ \cI \setminus \cC$ is empty or $\oest(\cC) \leq f(\mathcal{\cS})$}
		\State \textbf{return} $\cS$
		\Else
		\State $\mathbf{R}=\{ \mathcal{C} \cup \{Z\} \with Z \in \cI \setminus \cC \}$
		\State $\cC^* = \argmax \{ f(\cX') \with \cX' \in \mathbf{R} \}  $
		\State $\cX^*=\argmax  \{f(\cX') \with \cX' \in \{\cS,\cC^*\} \} $
		\State \textbf{return} $\BEAM(\cC^*,\cX^*) $
		\EndIf
		\EndFunction
		\State $\cX^*=\BEAM( \emptyset ,\emptyset) $
	\end{algorithmic}
\end{algorithm}

\section{Hardness of optimization}\label{sec:nphard}

In this section, we show that the problem of maximizing $\ourScore$ is
NP-hard by providing a reduction from the well-known NP-hard
\defemph{minimum set cover} problem: given a finite universe
$U=\{u_1,\dots,u_n\}$ and collection of subsets
$\mathcal{B}=\{B_1,\dots,B_m\} \subseteq 2^U$, find a set cover, i.e.,
a sub-collection $\cC \subseteq \cB$ with $\bigcup \cC=U$, that is of minimal
cardinality.

The reduction consists of two parts. First, we construct a base
transformation $\tau_1(U,\cB)=\mathbf{D}_{l}$ that maps a set
cover instance to a dataset $\mathbf{D}_{l}$ such that set covers
correspond to attribute sets with an empirical fraction of information score $\hNMI$ of
$1$ and bias correction terms $\hat{b}_0$ that are a monotonically
decreasing function of their cardinality. In a second step, we then
calibrate the $\hat{b}_0$ terms such
that, when considering the corrected score $\ourScore$, they cannot
change the order between attribute sets with different $\hNMI$ values
but only act as a tie-breaker between attribute sets of equal
$\hNMI$ value. This is achieved by copying the dataset $\mathbf{D}_l$
a suitable number of times $k$ such that the correction terms are
sufficiently small but the overall transformation, denoted
$\tau_{k}(U,\cB)=\mathbf{D}_{kl}$, is still of polynomial size.

\begin{figure}[t]
	\centering
	\begin{minipage}[c]{0.49\columnwidth}
		\includegraphics[width=\linewidth]{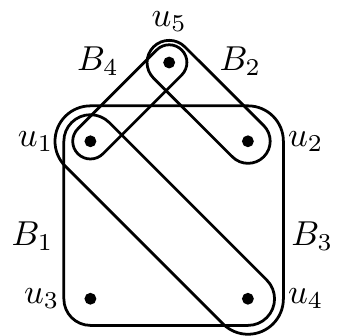}
	\end{minipage}
	\begin{minipage}[t]{0.49\columnwidth}
		\centering
		\footnotesize 
		\setlength{\tabcolsep}{3pt}
		
		\begin{tabular}{llccccr}
			\toprule
			
			& & $\mathbf{X_1}$ & $\mathbf{X_2}$ & $X_3$ & $X_4$ & $Y$ \\ \midrule
			& 1       & \textbf{1} & \textbf{a}     & 1     & 1     & a   \\
			& 2       & \textbf{a} & \textbf{2}     & 2     & a     & a   \\
			$S_1$	& 3       & \textbf{3} & \textbf{a}     & a     & a     & a   \\
			& 4       & \textbf{4} & \textbf{a}     & 4     & a     & a   \\
			& 5       & \textbf{a} & \textbf{5}     & a     & 5     & a   \\ \midrule
			& 6   & \textbf{a}     & \textbf{a}     & a     & a     & b   \\
			& 7   & \textbf{a}     & \textbf{a}     & a     & a     & b   \\
			$S_2$	& 8   & \textbf{a}     & \textbf{a}     & a     & a     & b   \\
			& 9   & \textbf{a}     & \textbf{a}     & a     & a     & b   \\
			& 10   & \textbf{a}     & \textbf{a}     & a     & a     & b   \\ \midrule
			& 11 & \textbf{b}     & \textbf{c}     & c     & c     & c   \\
			& 12 & \textbf{c}     & \textbf{b}     & c     & c     & c   \\
			$S_3$	& 13 & \textbf{c}     & \textbf{c}     & b     & c     & c   \\
			& 14 & \textbf{c}     & \textbf{c}     & c     & b     & c   \\ 
			& 15 & \textbf{c}     & \textbf{c}     & c     & c     & c   \\ 	\bottomrule
		\end{tabular}
	\end{minipage}
	\caption{\textbf{Base transformation example}. \textbf{Left:}  a set 
		cover instance $U=\{u_1,\hdots, u_5\}$ and $\cB=\{\mathbf{B_1},\mathbf{B_2},B_3,B_4\}$. \textbf{Right:} the resulting $\D_{15}$ using $\tau_1(U,\cB)$ (bold indicates the set cover) }
	\label{fig:initreduction}
\end{figure}
The \defemph{base transformation} $\tau_1(U,\cB)=\mathbf{D}_{l}$ is defined as
follows. 
The dataset $\mathbf{D}_{l}$ contains $m$
descriptive attributes $\cI=\{X_1, \dots, X_m\}$ corresponding to the
sets of the set cover instance, and a target variable Y. The sample size is $l=2n+m+1$ with a
logical partition of the sample into the three regions $S_1=[1,n]$,
$S_2=[n+1,2n]$, and $S_3=[2n+1,l]$.
The target attribute $Y$ assigns to sample points one of three values
corresponding to the three regions, i.e.,
$\mapping{Y}{\setupto{l}}{\{\text{a},\text{b},\text{c}\}}$ with
\begin{align*}
Y(j) &= \begin{cases}
\text{a} ,  & j \in S_1 \\
\text{b} , &  j \in S_2 \\
\text{c} , & j \in S_3 
\end{cases}\\
\intertext{and the descriptive attributes $X_i$ assign up to $n+3$
  distinct values dependending on the set of universe elements covered
  by set $B_i$, i.e., $X_i \with \setupto{l} \rightarrow \{1,2,\dots, n, \text{a}, \text{b}, \text{c}\}$ with}
X_{i}(j) &= \begin{cases}
j ,  &  j \in S_1 \wedge u_j \in B_i \\
\text{a} ,  &  (j \in S_1 \wedge u_j \not \in B_i ) \vee j \in S_2 \\
\text{b} , &  j=2n+i  \\
\text{c} , & j \in S_3 \setminus \{2n+i\} 
\end{cases} .
\end{align*}
See Fig.~\ref{fig:initreduction} for an illustration.
This transformation establishes a one-to-one correspondence of partial 
set covers $\cC \subseteq \cB$ and variable sets $\cX \subseteq \cI$, which we denote as $\cX(\cC)$. 
Let us denote $(\text{a},\hdots,\text{a})$ as $\vec{\text{a}}$.
The first part of the construction ($S_1$ and $S_2$) couples the amount of uncovered elements $U \setminus \bigcup \cC$ to the conditional entropy of $Y$ given $\cX(\cC)=\vec{\text{a}}$ through $\hat{p}(Y=\text{a}\given \cX(\cC)=\vec{\text{a}})=\card{U \setminus \bigcup \cC}/(n+\card{U \setminus \bigcup \cC})$.
The second part ($S_3$) links the size of $\cC$ to the number of distinct values on $S_3$.

We can note the following central properties.
\begin{lemma}\label{lem:t1}
Let $\tau_1(U,\cB)=\D_l$ be the transformation of a set cover instance $(U,\cB)$ and 
$\cC, \cC' \subseteq \cB$ be two partial set covers. Then the following holds.
\begin{enumerate}[label=\alph*)]
\item\label{t1:cover} If $\card{\bigcup \cC} \geq \card{\bigcup \cC'}$ then $\hNMI(\cX(\cC);Y) \geq \hNMI(\cX(\cC');Y)$;
in particular, $\cC$ is a set cover, i.e., $\bigcup \cC=U$, if and only if $\hNMI(\cX(\cC);Y)=1$,
\item \label{t1:mindiff} If $C$ is a set cover and $C'$ is \emph{not} a set cover then $\hI(\cX(\cC);Y)-\hI(\cX(\cC');Y) \geq 2/l$.
\item\label{t1:card} If $\cC$ and $\cC'$ are both set covers then $\homo{\cX(\cC)}{\cX(\cC')}$ 
if and only if $\card{\cC} \leq \card{\cC'}$.
\end{enumerate}
\end{lemma}
\begin{proof}
Statement \ref{t1:cover} follows from the definition of $\tau_1$.

To show \ref{t1:mindiff}, since $\hNMI(\cX(\cC');Y)$ and thus $\hI(\cX(\cC');Y)$
are monotone in $\card{\bigcup \cC'}$, it is sufficient to consider the case where
$\card{U \setminus \bigcup \cC'}=1$.
In this case we have
\[
\hI(\cX(\cC);Y)\!-\!\hI(\cX(\cC');Y)=\hH(Y \given \cX(\cC')) - \underbrace{\hH(Y \given \cX(\cC))}_{=0}\\
\]
and, moreover, as required
\begin{align*}
\hH(Y \given \cX(\cC')) &=-\hat{p}(\vec{\text{a}},\text{a})\log\hat{p}(\text{a}\given \vec{\text{a}})-\hat{p}(\vec{\text{a}},\text{b})\log\hat{p}(\text{b} \given \vec{\text{a}})\\
&=-\frac{1}{l}\log \left(\frac{1}{n+1} \right)-\frac{n}{l}\log\left(\frac{n}{n+1}\right) \geq \frac{2}{l} \enspace .
\end{align*}

For \ref{t1:card} observe that for a variable set $\cX=\cX(\cC)$ corresponding to
a set cover $\cC$, we have for all $i,j \in S_1$ that $\cX(i) \neq \cX(j)$. Thus, 
$\cX_{S_1} \equiv \cX'_{S_1}$ for a variable set $\cX'=\cX(\cC')$ corresponding to
another set cover $\cC'$. Moreover, we trivially have $\cX_{S_2} \equiv \cX'_{S_2}$.
Finally, let $Q,Q' \subseteq S_3$ denote the indices belonging to $S_3$ where $\cX$ and
$\cX'$ take on values different from $(c,\dots,c)$. Note that all values in these sets
are unique. Furthermore, if $\card{\cC} \leq \card{\cC'}$ then $\card{Q} \leq \card{Q'}$
and in turn $\card{Q \setminus Q'} \leq \card{Q' \setminus Q}$.
This means we can find a permutation $\sigma \in S_n$ such that for all $i \in Q \setminus Q'$
it holds that $\sigma(i)=j$ with $j \in  Q' \setminus Q$ and $\sigma(i)=i$ for $i \not\in Q \cap Q'$ (that is $\sigma$ permutes all indices of non-$(c,\dots,c)$ values of $\cC$ in $S_3$ to indices of non-$(c,\dots,c)$ values of $\cC'$).
For such a permutation it holds that $Y_\sigma\equiv Y$ and $\cX_{S_3} \preceq \cX'_{S_3 \sigma}$.
Therefore, $\homo{\cX}{\cX'}$ as required.
\end{proof}

Now, although set covers $\cC \subseteq \cB$ correspond to variable sets $\cX$ with the maximal empirical fraction of information
value of $1$, due to the correction term, it can happen that $\ourScore(\cX';Y)>\ourScore(\cX;Y)$
for a variable set $\cX'$ corresponding to a partial set cover.
To prevent this we make use of the following upper bound of the expected mutual information
under the permutation model.
\begin{proposition}[\cite{nguyen:2010:chancejournal}, Thm.~7]
\label{prop:upperbound}
For a sample of size $n$ of the joint distribution of variables $A$ and $B$ having 
$a$ and $b$ distinct values, respectively, we have
\[
\hat{m}_0(A,B,n) \leq \log\left(\frac{n+ab-a-b}{n-1}\right)
\]
\end{proposition}
This result implies that we can arbitrarily shrink the correction terms if we increase
the sample size but leave the number of distinct values constant.
Thus, we define the \defemph{extended transformation} $\tau_i(U,\cB)=\D_{il}$ through simply
copying $\D_l$ a number of $i$ times, i.e., by defining
$\bd_j=\bd_{(j \mod l)}$ for $j \in [l+1,il]$.
With this definition we show our main result.
\begin{theorem}
Given a sample of the joint distribution of variables $\cI$ and $Y$, the problem of
maximizing $\ourScore(\,\cdot\,;Y)$ over all subsets of $\cI$ is NP-hard.
\end{theorem} 
\begin{proof}
We show that there is a number $k \in O(l)$ such that w.r.t. transformation $\tau_k$
we have that for all set covers $\cC \subseteq \cB$ and their corresponding
variable sets $\cX=\cX(\cC)$, $\hat{m}_0(\cX,Y,n) < 2/l$. Since all properties of Lemma~\ref{lem:t1}
transfer from $\tau_1$ to $\tau_k$, this implies that for all variable sets $\cX'=\cX(\cC')$ 
corresponding to non-set-covers $\cC'\subseteq \cB$, it holds that
\begin{align*}
\ourScore(\X;Y) &= \hNMI(\cX;Y) - \hat{m}_0(\cX,Y,n)/\hH(Y) \\
&> \hNMI(\cX;Y) - 2/l\hH(Y)\\
&\geq \hNMI(\cX;Y) - (\hI(\cX;Y)-\hI(\cX';Y))/\hH(Y)\\
&= \hNMI(\cX';Y) \geq \ourScore(\X';Y)
\end{align*}
where the greater-than follows from Lm.~\ref{lem:t1}\ref{t1:cover} and \ref{lem:t1}\ref{t1:mindiff}.
Thus, any $\cX$ with maximum $\ourScore$ corresponds to a set cover $\cC$.

Moreover, we know that $\cC$ must be a minimum set cover as required, because for a smaller
set cover $\cC'$, $\homo{\cX(\cC') }{\cX(\cC)}$ by Lemma~\ref{lem:t1}\ref{t1:card}
and thus $\hat{b}_0(\cX(\cC'),Y,n) \leq \hat{b}_0(\cX(\cC),Y,n)$ from Prop.~\ref{prop:isospec}\ref{isospec:correction}---therefore, $\cX(\cC)$ would not maximize $\ourScore$.

To find the number $k$ that defines the final transformation $\tau_k$,
let $\D_{il}=\tau_{i}(U,\cB)$ and $\cC$ be a set cover of $(U,\cB)$.
Since $\cX=\cX(\cC)$ has at most $3l$ distinct values in $\D_{il}$ and $Y$
exactly $3$, from Prop.~\ref{prop:upperbound} and the monotonicity of $\ln$ we know that
\[
\ln(2)\hat{m}_0(\cX(\cC),Y,n) \leq \ln \left(\frac{il+3l}{il-1}\right) \leq \ln \left(\frac{i+3}{i-1}\right) \leq \frac{4}{i-1}
\]
where the last inequality follows from $\ln(x)\leq x-1$.
Thus, for $k=\left\lceil 2l/\ln 2 \right\rceil+1 \in O(l)$ we have $\hat{m}_0(\cX,Y,n) < 2/l$
as required. The proof is concluded by noting that the final transformation $\tau_k(U,\cB)$
is of size $O(l^2m)$ (where $l=2n+m+1$), which is polynomial in the size of the set cover
instance.
\end{proof}

	\section{Refined bounding function}\label{sec:oest}

The NP-hardness established in the previous section excludes (unless P=NP) the existence of a polynomial time algorithm for maximizing the reliable fraction of information, leaving therefore exact but exponential search and heuristics as the two options. For both, and particularly the former, reducing the search space can lead to more effective algorithms. For this purpose, we derive in this section a novel bounding function for $\ourScore$ to be used for pruning.


\begin{figure*}[t]
	\centering
	\begin{minipage}[b]{\columnwidth}
		\centering
			\includegraphics[width=\linewidth]{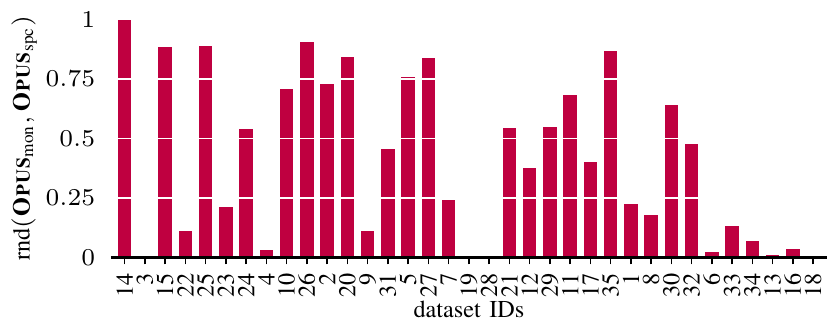}
		\label{fig:oestn} 
	\end{minipage}
	\begin{minipage}[b]{\columnwidth}
		\centering
		\includegraphics[width=\linewidth]{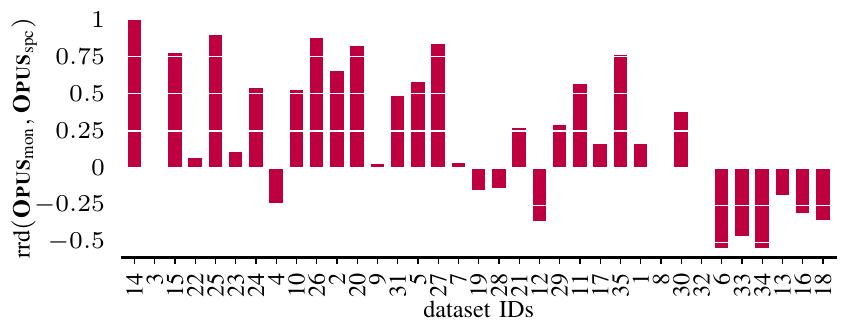}
	\end{minipage}
	\caption{\textbf{Evaluating $\fbarn$ for branch-and-bound optimization.} Relative nodes explored difference (left) and relative runtime difference (right) between methods $\methodA$ and $\methodB$. Positive (negative) numbers indicate that $\methodA$ ($\methodB$) is proportionally ``better". The datasets are sorted in decreasing number of attributes.}
	\label{fig:oesto}
\end{figure*}

Recall that an admissible bounding function $\oest$ for effective search is an upper bound to the optimization function value $f$ of all supersets of a candidate solution $\cX \subseteq \cI$.
That is, it must hold that $\oest (\cX) \geq f(\cX')$ for all $\cX'$ with $\cX \subseteq \cX' \subseteq \cI$.
At the same time, in order to yield optimal pruning, the bound should be as tight as possible.
Thus, the ideal function is
\begin{equation}
\bar{f}_{\text{ideal}}(\cX)=\max \{\ourScore(\cX';Y) \with \cX \subseteq \cX' \subseteq \cI \}  \enspace.
\end{equation}
Computing this function is of course equivalent to the original optimization problem and hence NP-hard. We can, however, relax the maximum over all supersets to the maximum over all \emph{specializations} of $\cX$.
That is, we define a bounding function $\fbarn(\cX)$ through 
\begin{align}
\fbarn(\cX) =& \max\{\ourScore(\cX';Y) \with \cX \preceq \cX'\} \\ 
\geq & \max  \{\ourScore(\cX';Y) \with \cX \subseteq \cX' \subseteq \cI\} 
= \bar{f}_{\text{ideal}}(\cX) \enspace.
\end{align}

While this definition obviously constitutes an admissible bounding function, it is unclear how it can be efficiently evaluated.
Let us denote by $R^+$ the operation of joining a labeling $R$ with the target attribute $Y$,  i.e., $R^+=\{R\} \cup \{Y\}$ (see Tab.~\ref{tab:xp} for an example).
This definition gives rise to a simple constructive form for computing $\fbarn$.
\begin{theorem}\label{lem:xy}
	The function $\fbarn$ can be efficiently computed as $\fbarn(\cX) = \ourScore(\cXp;Y)$ in time $O(n\card{\domain{\cX}}\card{\domain{Y}})$.
\end{theorem}
\begin{proof}
	We start by showing that the $(\cdot)^+$ operation causes a positive gain in $\ourScore$, i.e., for an arbitrary labeling  $R$ it holds that $\ourScore(R^+;Y) \geq  \ourScore(R;Y)$.

	Let us define $\Io(R^+;Y)=\hI(R^+;Y)-\hat{E}_0(\hI(R^+,Y))$. It is then sufficient to show $\Io(R^+;Y) \geq \Io(R;Y)$. We have
	\begin{align*}
	\Io(R^+;Y)=&\left(  \hH(Y)+\hH(R^+)-\hH(R^+,Y) \right)\\
	-&\frac{1}{n!} \left( \sum_{\sigma \in S_n}(\hH(Y_\sigma)+\hH(R^+)-\hH(R^+,Y_\sigma)\right) \\
	=&\frac{1}{n!}\sum_{\sigma \in S_n} \hH(R^+,Y_\sigma) - \hH(R^+,Y) \\
	\geq & \frac{1}{n!} \sum_{\sigma \in S_n}\hH(R,Y_\sigma) - \hH(R,Y) = \Io(R;Y) \enspace,
	\end{align*}
	since $\hH(R^+,Y)=\hH(R \cup Y,Y)=\hH(R,Y)$, and from Prop.~\ref{prop:spec}\ref{isospec:correction},  for every $\sigma \in S_n$, $\hH(R^+,Y_\sigma) \geq \hH(R,Y_\sigma)$.
	
	To conclude, let $\cZ$ be an arbitrary specialization of  $\cX $. We have by definition of $\cZ$ and $\cZ^+$, that $\cXp \preceq \cZ^+$. Moreover, $\hNMI(\,\cdot\,;Y)=\hNMI(\{\,\cdot\,\} \cup \{Y\};Y)=1$. Thus
	\begin{align*}
	\ourScore(\cXp;Y)=&\hNMI(\cXp;Y)-\bo(\cXp,Y,n) \\ 
	=& 1-\bo(\cXp,Y,n) \\
	\geq & 1- \bo(\cZ^+,Y,n) \\
	= & \ourScore(\cZ^+;Y) \geq \ourScore(\cZ;Y) \enspace,
	\end{align*}
	as required. Here, the first inequality follows from Prop.~\ref{prop:spec}\ref{isospec:correction}, the second from the positive gain of $\cZ^+$ over $\cZ$. 
	
	For the complexity recall that $\bo(\cX,Y,n)$ can be computed in time $O(n \max\{\card{\domain{\cX}}, \card{\domain{Y}}\})$ (see appendix).
	The complexity follows from $\card{\domain{\cXp}} \leq \card{\domain{\cX}} \card{\domain{Y}}$.
\end{proof}

Note that the $\cX^+$ operation does not have to be computed explicitly because the non-zero marginal counts for $\cX^+$ can simply be obtained as the non-zero counts of the joint contingency table of $\cX$ and $Y$ (which has to be computed anyway for $\ourScore$; see appendix).

Intuitively, $\cX^+$ constitutes the most efficient specialization of $\cX$ in terms of growth in $\hNMI$ and $\bo$ (which is not necessarily attainable by a subset of input variables).
In contrast, the bounding function $\fbaro(\cX)=1-\bo(\cX,Y,n)$ of~\cite{mandros:2017:discovering} assumes that full information about the target can be attained (i.e., $\hNMI=1$) without ``paying" an increased $\bo$ term. The following proposition shows that this idea leads to an inferior bound.

\begin{proposition}\label{theo:maximizer}
	Let $\cX \subseteq \cI$ and $\Delta=\fbaro(\cX)-\fbarn(\cX)$. The following statements hold:
	\begin{enumerate}[label=\alph*)]
		\item $\Delta \geq 0$ for all datasets, i.e., $\fbarn(\cX) \leq \fbaro(\cX)$ 
		\item there are datasets $\D_{4l}$ for all $l \geq1$ s.t. $\Delta \in \Omega(1-\frac{1}{\log 2l})$
	\end{enumerate}
\end{proposition} 
\begin{proof}
	a)
	\begin{align*}
	\fbarn(\cX)=&1-\bo(\cXp,Y,n) \\ 
	\leq  & 1 - \bo(\cX,Y,n) =\fbaro(\cX) \enspace,
	\end{align*}
	where the inequality holds from  Prop.~\ref{prop:spec}\ref{isospec:correction} and $\cX \preceq \cXp$. 
	
	b) For $l\geq1$ we construct a dataset $\D_{4l}$ with two variables $X \with [4l] \rightarrow \{\text{a},\text{b}\}$ and $Y \with \setupto{4l} \rightarrow \setupto{2l}$, with 
	\begin{align*}
	X(i) &= \begin{cases}
	\text{a} ,  & i \mod 2=1\\
	\text{b} ,  & i \mod 2=0
	\end{cases}
	\end{align*}
	and $Y(i)= \ceil{i/2}$  respectively (see Tab.~\ref{tab:xp}). We have 
	\begin{align*}
	\Delta&= 1-\bo(X,Y,4l)-1+\underbrace{\bo(X^+,Y,4l)}_{=\hH(Y\given X_\sigma^+)/\hH(Y)=0} \\
	&=\frac{1}{n!}\sum_{\sigma \in S_n} \hH(Y \given X_\sigma) / \hH(Y) \\
	&\geq \min_{\sigma \in S_n} \hH(Y \given X_\sigma) / \hH(Y) \enspace .
	\end{align*}
One can show that the minimum of the last step is attained by the permutation $\sigma^* \in S_n$ with
\begin{align*}
	\sigma^*(i) &= \begin{cases}
	2i-1, & i \in [1,2l]\\
	4l-2(4l-i),   & i \in [2l+1,4l]
	\end{cases} \enspace ,
\end{align*}
which corresponds to sorting the $\text{a}$ and $\text{b}$ values of $X$ (see Tab.~\ref{tab:xp}).
For this permutation the normalized conditional entropy evaluates to $1-1/\log(2l)$ as required. 
\end{proof}

\begin{table}[t]
	\begin{minipage}[c]{0.49\columnwidth}
		\centering
		\begin{tabular}{cccc}
			\toprule
			$X$ & $Y$      & $X^+$   & $X_{\sigma^*}$ \\ \midrule
			a   & 1        & \hphantom{0.5}(a,1)\hphantom{0.3}& a              \\
			b   & 1        & $($b,1$)$ & a              \\
			a   & 2        & $($a,2$)$ & a              \\
			b   & 2        & $($b,2$)$ & a              \\
			\multicolumn{4}{c}{$\vdots$}         \\ \bottomrule
		\end{tabular}
	\end{minipage}
	\hfill
	\begin{minipage}[c]{0.49\columnwidth}
		\centering
		\begin{tabular}{cccc}
			\toprule
			$X$ & $Y$      & $X^+$   & $X_{\sigma^*}$ \\ \midrule
			\multicolumn{4}{c}{$\vdots$}     \\
			a   & 2l-1        & $($a,2l-1$)$ & b              \\
			b   & 2l-1        & $($b,2l-1$)$ & b              \\
			a   & 2l      & $($a,2l$)$ & b              \\
			b   & 2l        & $($b,2l$)$ & b              \\ \bottomrule	\end{tabular}		
	\end{minipage}
	\caption{\textbf{Construction showing advantage of upper bound $1-\bo(X^+,Y,n)=0$ vs $1-\bo(X,Y,n) \geq 1-1/\log(n/2)$}, i.e., all specializations of $X$ that contain full information about $Y$ are injective (key) maps (see Prop.~\ref{theo:maximizer}).}
	\label{tab:xp}
\end{table}

Thus, we have established that $\fbarn$ is not only tighter than $\fbaro$, but even that the difference can be arbitrary close to $1$ (for an increasing domain size of $Y$). Put differently, their ratio, and thus the potential for additional pruning, is unbounded.

Computationally, $\fbarn(\cX)$ is more expensive than $\fbaro(\cX)$ by a factor of $\card{\domain{Y}}$.
In order to partially alleviate this increase, note that one can first check the pruning condition (line 8 of Alg.~\ref{alg:opus} or line 2 of Alg.~\ref{alg:greedy}) w.r.t. $\fbaro$ and only compute $\fbarn$ if that first check fails.
That is, whenever $\fbaro(\cX)$ is sufficient to prune a candidate $\cX$ we can still do so with the same computational complexity.
However, the additional evaluation of $\fbarn(\cX)$ can be a disadvantage in case it still does not allow to prune. 
This trade-off is evaluated in the following section. 

	\section{Evaluation}\label{sec:eval}
\begin{figure}[t]
	\centering
	\begin{minipage}[b]{\columnwidth}
		\centering
					\includegraphics[width=\linewidth]{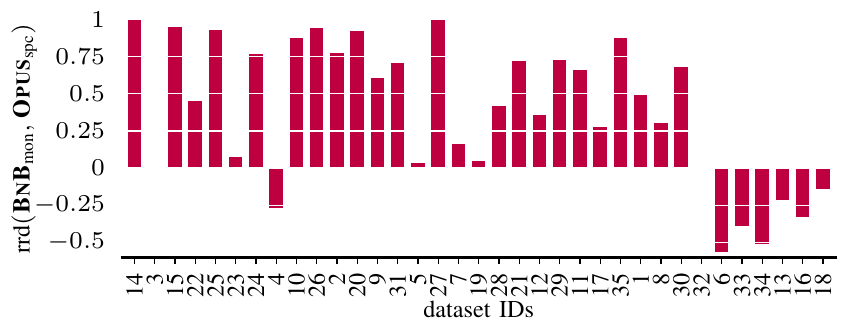}
	\end{minipage}
	\caption{\textbf{Evaluating the two branch-and-bound frameworks.} Relative time difference between methods $\methodA$ and $\methodC$. Positive (negative) numbers indicate that $\methodA$ ($\methodC$) is proportionally ``better".  Datasets are sorted in decreasing number of attributes.}
	\label{fig:ovsn}
\end{figure}

For ease of comparison to \cite{mandros:2017:discovering}, we consider datasets from the KEEL data repository~\cite{keel}. In particular, we use all classification datasets with $d \in [10,90]$ and no missing values, resulting in $35$ datasets with $52000$ and $30$  rows and columns on average, respectively. All metric attributes are discretized in 5 equal-frequency bins. The datasets are summarized in Table~\ref{tab:data}.  The runtimes are averaged over $3$ runs. All implementations are available online\!\!~\footnote{\giturl}\!.

We use two metrics for evaluation, the relative runtime difference and the relative difference in number of explored nodes. For methods A and B, the relative runtime difference on a particular dataset is computed as
\[
\text{rrd}(A, B)= \frac{(\tau_{A}-\tau_{B})}{\max(\tau_{A},\tau_{B})} \enspace,
\]
where $\tau_{A}$ and $\tau_{B}$ are the run times for A and B respectively. The $\text{rrd}$ score lies in $[-1,1]$, where positive (negative) values indicate that B is proportionally faster (slower). For example, a $\text{rrd}$ score of $0.5$ corresponds to a factor of $2$ speed-up, $0.66$ to a factor of $3$, $0.75$ to $4$ etc. The relative nodes explored difference $\text{rnd}$ is defined similarly. For both scores, we consider $(-0.5,0.5)$ to be a region of practical equivalence, i.e., a factor of $2$ of improvement is required to consider a method ``better".

\subsection{Branch-and-bound}\label{sec:bnb}

We first investigate the effect of the refined bounding function by comparing $\methodA$ and $\methodB$, i.e., Alg.~\ref{alg:opus} with $\fbarn$ and $\fbaro$ as bounding functions respectively. Last, we compare the proposed branch-and-bound framework $\methodA$ to the one of~\cite{mandros:2017:discovering}, which we call $\methodC$ (a combination of best-first search and refinement based on lexicographical order).  For a fair comparison, we set a common $\alpha$ value for all three methods on each dataset by determining the largest $\alpha$ value in increments of $0.05$ such that they terminate in less than $90$ minutes (see Tab.~\ref{tab:data}). 

In Fig.~\ref{fig:oesto} we present the comparison between $\methodA$ and $\methodB$. The left plot demonstrates that $\fbarn$ can lead to a considerable reduction of nodes explored over $\fbaro$. In particular, $15$ cases have at least a factor of $2$ reduction, $7$ have $4$, and there is one $1$ with $760$. For $20$ cases there is no practical difference. The plot validates that the potential for additional pruning is indeed unbounded~(Sec.~\ref{sec:oest}).  In terms of runtime efficiency (right plot), $\methodA$ is ``faster" in $70\%$ of the datasets. In more detail, and considering practical improvements, $12$ datasets have at least a factor of $2$ speedup, $6$ have $4$, $1$ has $266$, while only $2$ have a factor of $2$ slowdown. Moreover, we observe from the plot (where datasets are sorted in decreasing number of attributes) a clear correlation between number of attributes and efficiency: the $6$ out of $10$ datasets with the slowdown are also the ones with the lowest number of features. Overall, $\fbarn$ leads to a more effective optimization with branch-and-bound, and particularly for the higher-dimensional cases.

Following, we compare $\methodA$ to $\methodC$, presenting the results in Fig.~\ref{fig:ovsn}. The plot is quite evident that the new framework outperforms the baseline. There are $16$ cases with at least a speedup of $2$x, $10$ with $4$x, and there even exists a case with $2880$x. Moreover, since both $\methodB$ and $\methodC$ use the same bounding function $\fbaro$, the two plots Fig.~\ref{fig:ovsn} and Fig.~\ref{fig:oesto} (right) suggest that Alg~\ref{alg:opus} is a more effective branch-and-bound framework for the reliable fraction of information.

\subsection{Greedy}\label{sec:greedyeval}
\begin{figure}[t]
	\centering
	\begin{minipage}[b]{\columnwidth}
		\centering
						\includegraphics[width=\linewidth]{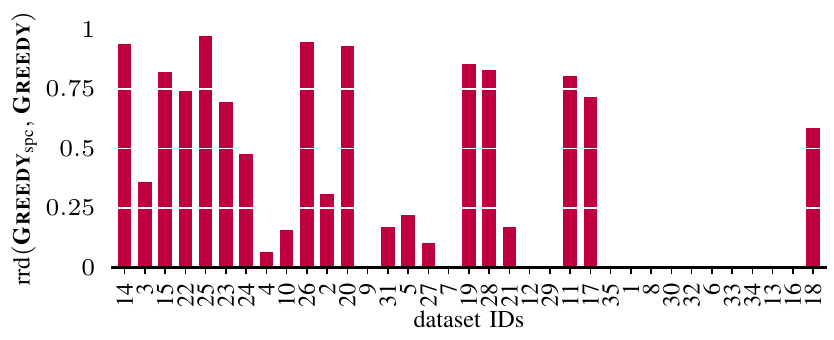}
	\end{minipage}
	\caption{\textbf{Evaluating $\fbarn$ for heuristic optimization.} Relative time difference between methods $\methodD$ and $\methodE$. Positive (negative) numbers indicate that $\methodD$ ($\methodE$) is proportionally ``better". The datasets are sorted in decreasing number of attributes.}
	\label{fig:oestg}
\end{figure}

We begin the evaluation with the performance of $\fbarn$ for heuristic search. We present the relative runtime differences of $\methodE$ and $\methodD$, i.e., Alg~\ref{alg:greedy} with and without $\fbarn$, in Fig.~\ref{fig:oestg} (raw results in Table~\ref{tab:data}). The plot shows that $\fbarn$ indeed improves the efficiency of the heuristic search, as we find that for $12$ datasets there is a speedup of at least a factor of $2$, and $8$ of at least a factor of $4$. 		

Next, we investigate the quality of the greedy results. Note that this is possible as we have access to the branch-and-bound results. In Fig.~\ref{fig:gvsoquality} we plot the differences between the $\ourScore$ score of the results obtained by greedy and branch-and-bound on each dataset (raw results in Table~\ref{tab:data}). Note that branch-and-bound uses the same $\alpha$ values as with the experiments in Sec~\ref{sec:bnb}, and that we only plot the non-zero differences in the two plots, left for $\alpha=1$, i.e, optimal solutions, and right for $\alpha<1$, i.e., approximate solutions with guarantees. 

At a first glance, we observe that there is no difference in $21$ out of $35$ cases considered, $7$ where greedy is better (this of course on the datasets where $\alpha <1$), and $7$ for branch-and-bound. Out of the $21$ cases where the two algorithms have equal $\ourScore$, $16$ of them have $\alpha=1$, i.e., the greedy algorithm is optimal roughly $45\%$ of the time. Moreover, the cases where branch-and-bound is better is only by a small margin, $0.03$ on average, while greedy ``wins" by $0.1$ on average. Another observation from the right plot of Fig.~\ref{fig:gvsoquality} is that the largest differences between the two algorithms is for the $3$ datasets where the lowest $\alpha$ values where used, i.e., $0.05, 0.1,$ and $0.35$. 

Lastly in Fig~\ref{fig:ovsbtime}, we consider the relative runtime difference between the greedy algorithm and branch-and-bound, i.e., $\methodD$ and $\methodA$. As expected, the greedy algorithm is significantly faster in the majority of cases. There are, however, $4$ cases where branch-and-bound terminates much faster, which also happen to  coincide with more aggressive $\alpha$ values for branch-and-bound. 

\begin{figure}[t]
	\centering
	\begin{minipage}[b]{0.49\columnwidth}
		\centering
							\includegraphics[width=\linewidth]{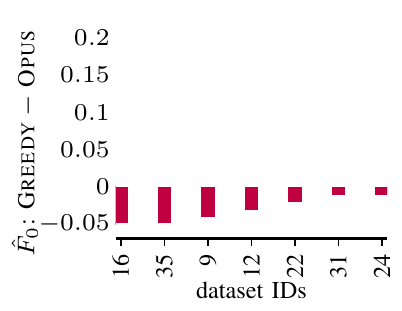}
	\end{minipage}
	\begin{minipage}[b]{0.49\columnwidth}
		\centering
				\includegraphics[width=\linewidth]{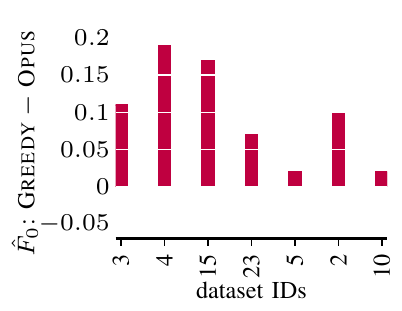}
	\end{minipage}
	\caption{\textbf{Evaluating the heuristic algorithm for result quality.} \textbf{Left}: difference in $\ourScore$ between methods $\methodD$ and $\methodA$ (i.e., $\ourScore(\cXo_{grd};Y)-\ourScore(\cXo_{bnb};Y)$ where $\cXo_{grd}$ and $\cXo_{bnb}$ are the solutions of Alg.~\ref{alg:greedy} and~\ref{alg:opus} respectively)  for $\alpha=1$. Since $\alpha=1$, the negative values close to $0$ indicate that Alg.~\ref{alg:greedy} retrieves nearly optimal solutions. Data are sorted in increasing quality difference. \textbf{Right}: difference for $\alpha<1$. Positive values indicate that Alg.~\ref{alg:greedy} retrieves better solutions when Alg.~\ref{alg:opus} uses guarantees $\alpha<1$.   Data are sorted in increasing $\alpha$ values.}
	\label{fig:gvsoquality}
\end{figure}

\begin{figure}[t]
	\centering
	\begin{minipage}[b]{\columnwidth}
		\centering
									\includegraphics[width=\linewidth]{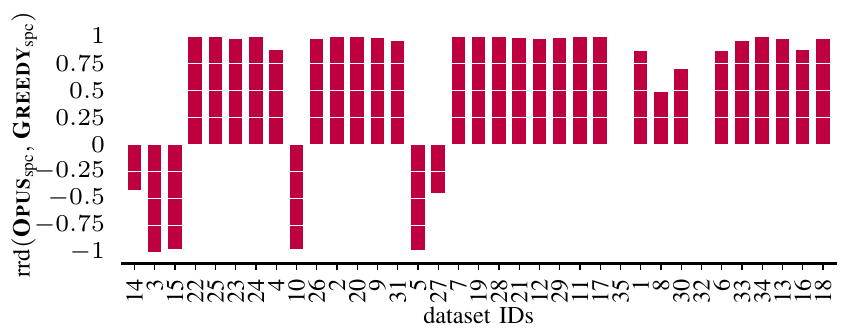}
	\end{minipage}
	\caption{\textbf{Evaluating the heuristic algorithm in terms of running time.} Relative time difference between methods $\methodD$ and $\methodA$. Positive (negative) numbers indicate that $\methodD$ ($\methodA$) is proportionally ``better". Datasets are ordered in decreasing number of attributes.}
	\label{fig:ovsbtime}
\end{figure}

\begin{sidewaystable*}[t]
	\small
	\centering
	\caption{Datasets used in Section~\ref{sec:eval} along with the raw results produced from the evaluation.}
	\label{tab:data}%
	\rowcolors{2}{gray!25}{white}
	\begin{tabular}{rlrrrrrrrrrrrrr}
		&       &       &       &       &        \multicolumn{6}{c}{time(s)}           & \multicolumn{2}{c}{nodes explored} &  \multicolumn{2}{c}{$\ourScore$} \\ \toprule
		ID & dataset & \#rows & \#attr & \#classes & $\alpha$ & $\methodA$ & $\methodB$ & $\methodC$ & $\methodD$ & $\methodE$ & $\methodA$ & $\methodB$ & $\methodA$ & $\methodD$ \\ \bottomrule
		1 & \textit{australian} & $690$   & $14$    & $2$     & $1.00$     & $7.0$     & $8.3$   & $13.7$  & $1.0$     & $1.0$     & $4190$  & $5388$  & $0.54$  & $0.54$ \\
		2 & \textit{chess} & $3196$  & $36$    & $2$     & $0.75$  & $192.1$ & $545.9$ & $828.4$ & $2.5$   & $3.6$   & $69713$ & $252766$ & $0.77$  & $0.87$ \\
		3& \textit{coil2000} & $9822$  & $85$    & $2$     & $0.05$  & $1.0$     & $1.0$     & $1.0$     & $189.1$ & $294.4$ & $86$    & $86$    & $0.06$  & $0.17$ \\
		4 & \textit{connect-4} & $67557$ & $42$    & $3$     & $0.10$   & $1236.8$ & $951.5$ & $914.3$ & $164.3$ & $174.8$ & $36183$ & $37176$ & $0.10$   & $0.29$ \\
		5 & \textit{fars}  & $100968$ & $29$    & $8$     & $0.65$  & $3.0$     & $7.0$     & $3.1$   & $93.9$  & $119.8$ & $45$    & $183$   & $0.66$  & $0.68$ \\
		6 & \textit{flare} & $1066$  & $11$    & $6$     & $1.00$     & $6.8$   & $3.2$   & $3.0$     & $1.0$     & $1.0$     & $2011$  & $2048$  & $0.65$  & $0.65$ \\
		7 & \textit{german} & $1000$  & $20$    & $2$    & $1.00$     & $931.5$ & $960.1$ & $1104.9$ & $1.0$     & $1.0$     & $216250$ & $284397$ & $0.21$  & $0.21$ \\
		8 & \textit{heart} & $270$   & $13$    & $2$    & $1.00$     & $1.9$   & $1.9$   & $2.7$   & $1.0$     & $1.0$     & $2275$  & $2758$  & $0.42$  & $0.42$ \\
		9 & \textit{ionosphere} & $351$   &$33$    & $2$    & $1.00$     & $46.4$  & $47.6$  & $116.4$ & $1.0$     & $1.0$     & $48094$ & $53784$ & $0.62$  & $0.58$ \\
		10 & \textit{kddcup} & $494020$ & $41$    & $23$    & $0.90$   & $18.1$  & $37.8$  & $142.3$ & $520.2$ & $616.4$ & $69$    & $232$   & $0.97$  & $0.99$ \\
		11 & \textit{letter} & $20000$ & $16$    & $26$    & $1.00$     & $659.5$ & $1501.0$  & $1915.5$ & $3.8$   & $19.1$  & $4894$  & $15300$ & $0.60$   & $0.60$ \\
		12 & \textit{lymphography} & $148$   & $18$    & $4$    & $1.00$     & $31.2$  & $20.2$  & $48.4$  & $1.0$     & $1.0$     & $23971$ & $38319$ & $0.48$  & $0.45$ \\
		13 & \textit{magic} & $19020$ & $10$    &$2$     & $1.00$     & $38.5$  & $31.6$  & $30.4$  & $1.3$   & $1.3$   & $1012$  & $1017$  & $0.43$  & $0.43$ \\
		14 & \textit{move-libras} & $360$   & $90$   & $15$    & $0.50$   & $1.0$     & $266.6$ & $2881.7$ & $1.7$   & $25.9$  & $213$   & $163630$ & $0.32$  & $0.32$ \\
		15 & \textit{optdigits} & $5620$  & $64$    & $104$    & $0.35$  & $1.0$     & $4.3$   & $18.9$  & $25.1$  & $139.3$ & $105$   & $888$   & $0.36$  & $0.53$ \\
		16 & \textit{pageblocks} & $5472$  & $10$    & $5$    & $1.00$     & $7.4$   & $5.2$   & $5.0$     & $1.0$     & $1.0$     & $831$   & $859$  & $0.65$  & $0.60$ \\
		17 & \textit{penbased} & $10992$ & $16$    & $10$    & $1.00$     & $233.6$ & $277.5$ & $320.9$ & $1.6$   & $5.6$   & $8099$  & $13486$ & $0.75$  & $0.75$ \\
		18 & \textit{poker}  & $1025010$ & $10$    & $10$    & $1.00$     & $2594.7$ & $1705.2$ & $2247.4$ & $86.0$    & $205.3$ & $908$   & $908$   & $0.57$  & $0.57$ \\
		19 & \textit{ring}  & $7400$  & $20$    & $2$    & $1.00$     & $1393.9$ & $1197.3$ & $1456.7$ & $1.1$   & $7.4$   & $60460$ & $60460$ & $0.29$  & $0.29$ \\
		20 & \textit{satimage} & $6435$  & $36$    & $7$     & $0.80$   & $173.8$ & $954.4$ & $2054.2$ & $2.0$     & $27.6$  & $24622$ & $154287$ & $0.74$  & $0.74$ \\
		21 & \textit{segment} & $2310$  & $19$    & $7$    & $1.00$     & $39.1$  & $53.3$  & $137.2$ & $1.0$     & $1.2$   & $8815$  & $19159$ & $0.84$  & $0.84$ \\
		22 & \textit{sonar} & $208$  & $60$    & $2$     & $1.00$     & $403.5$ & $431.9$ & $730.8$ & $1.0$     & $3.8$   & $472554$ & $530478$ & $0.34$  & $0.32$ \\
		23 & \textit{spambase} & $4597$  & $57$    & $2$     & $0.55$  & $515.6$ & $574.6$ & $555.4$ & $15.4$  & $50.1$  & $167695$ & $212147$ & $0.54$  & $0.60$ \\
		24 & \textit{spectfheart} & $267$   & $44$   & $2$    & $1.00$     & $171.1$ & $369.3$ & $724.7$ & $1.0$     & $1.9$   & $155364$ & $335365$ & $0.23$  & $0.22$ \\
		25 & \textit{splice} & $3190$ & $60$    & $3$     & $0.65$  & $92.3$  & $851.0$   & $1193.0$  & $1.5$   & $46.9$  & $36052$ & $315125$ & $0.65$  & $0.65$ \\
		26 & \textit{texture} & $5500$  & $40$    & $11$    & $0.80$   & $62.9$  & $480.3$ & $1000.8$ & $2.1$   & $36.6$  & $10835$ & $109878$ & $0.76$  & $0.76$ \\
		27 & \textit{thyroid} & $7200$  & $21$    & $3$     & $0.50$   & $1.0$     & $5.8$   & $110.4$ & $1.8$   & $2.0$     & $247$   & $1475$  & $0.50$   & $0.50$ \\
		28 & \textit{twonorm} & $7400$  & $20$   & $2$     & $1.00$     & $1332.2$ & $1162.4$ & $2274.4$ & $1.3$   & $7.4$   & $60460$ & $60460$ & $0.42$  & $0.42$ \\
		29 & \textit{vehicle} & $846$   & $18$    & $4$     & $1.00$     & $38.2$  & $53.2$  & $137.9$ & $1.0$     & $1.0$     & $10670$ & $23462$ & $0.48$  & $0.48$ \\
		30 & \textit{vowel} & $990$  & $13$    & $11$    & $1.00$     & $3.2$   & $5.1$   & $9.8$   & $1.0$     & $1.0$     & $590$   & $1622$  & $0.45$  & $0.45$ \\
		31 & \textit{wdbc}  & $569$   & $30$   & $2$     & $1.00$     & $19.9$  & $38.2$  & $67.0$    & $1.0$     & $1.2$   & $18564$ & $33862$ & $0.76$  & $0.75$ \\
		32 & \textit{wine}  & $178$   & $13$    & $3$     & $1.00$     & $1.0$     & $1.0$     & $1.0$     & $1.0$     & $1.0$     & $199$   & $378$   & $0.71$  & $0.71$ \\
		33 & \textit{wine-red} & $1599$  & $11$    & $11$    & $1.00$     & $18.7$  & $10.3$  & $11.5$  & $1.0$     & $1.0$     & $1481$  & $1698$  & $0.20$   & $0.20$ \\
		34 & \textit{wine-white} & $4898$  & $11$    & $11$    & $1.00$     & $77.4$  & $36.2$  & $38.1$  & $1.0$     & $1.0$     & $1881$  & $2011$  & $0.19$  & $0.19$ \\
		35 & \textit{zoo}   & $101$   & $15$    & $7$    & $1.00$     & $1.0$     & $4.1$   & $7.7$   & $1.0$     & $1.0$     & $773$   & $5724$  & $0.80$   & $0.75$ \\ \bottomrule
		\textbf{Avg.} &       &  &  &    &  & $296$ & $360$ & $603$ & $32$ & $51$& $41434$ & $78309$ & $0.51$ & $0.53$ \\ 
		\bottomrule
	\end{tabular}%
\end{sidewaystable*}%

	\section{Conclusion}\label{sec:conc}
We investigated the algorithmic aspect of discovering dependencies in data using the reliable fraction of information, where we proved the NP-hardness of the problem and derived a refined bounding function for more effective optimization. Moreover, we considered an improved branch-and-bound algorithm and explored the aspects of heuristic optimization. The experimental evaluation showed that the refined bounding function is very effective for both types of optimization, the proposed branch-and-bound framework outperforms the baseline, and that the greedy algorithm provides solutions that are nearly optimal.

While the given reduction from set cover can be extended to show that, unless P=NP, no fully polynomial time approximation scheme exists, the possibility of weaker approximation guarantees remains. 
In particular, the strong empirical performance of the greedy algorithm hints that $\ourScore$ could have a certain structure favored by the greedy algorithm, e.g., some weaker form of submodularity (we remind that $\ourScore$ is neither submodular nor monotone). For instance, we could explore ideas from Horel and Singer~\cite{horel:2016:sub} where a monotone function is e-approximately submodular if it can be bounded by a submodular function within $1 \pm e$. Another idea is that of restricted submodularity for monotone functions~\cite{du:2008:restricted}, where a function  is submodular over a subset of the search space. It might be that the greedy algorithm only considers candidates where $\ourScore$ is submodular.

Furthermore, the proposed bounding function is likely to be applicable to a larger selection of corrected-for-chance dependence measures, and a general framework for maximizing reliable measures could be established.


\bibliographystyle{IEEEtran}

\section*{Appendix}
\newcommand{\ppm}[1]{\hat{p}_{o}(#1)}
Here, we recap how to efficiently compute the correction term $\mo(\cX,Y,n)$ for an attribute set $\cX \subseteq \cI$ and target $Y$ in $\D_n$ (cf. \cite{mandros:2017:discovering, romano:2014:smi}).
Let the observed domains of $\X$ and $Y$ be $\domain{\X}=\lbrace \vx_1,\dots, \vx_R \rbrace$ and $\domain{Y}=\lbrace y_1,\dots, y_C \rbrace$, respectively.
We define shortcuts for the observed marginal counts $a_i=c(\X=\vx_i)$ and $b_j=c(Y=y_j)$ as well as for the joint counts $c_{i,j}=c(\X=\vx_i, Y=y_j)$. The \defemph{contingency table} $\vc$ for $\cX$ and $Y$ is then the complete joint count configuration $\vc=\{c_{i,j} \with 1\leq i \leq R, 1 \leq j \leq C\}$. The empirical mutual information for $\cX$ and $Y$ can then be computed as:
\begin{equation}
	\hI(\cX,Y)=\hI(\vc)=\sum_{i=1}^{R} \sum_{j=1}^{C}\frac{c_{ij}}{n} \log \frac{ c_{ij} n} {a_i b_j}
\end{equation}

Each $\sigma \in S_n$ results in a contingency table $\vc^{\sigma}$. 
We denote with $\cT=\{\vc^{\sigma} \with \sigma \in S_n\}$ the set of all such contingency tables. 
Crucially, all these tables have the same marginal counts $a_i,b_j$, $i \in [1, R], j \in [1, C]$. Hence, one can rewrite $\mo$ as 
\begin{equation}
\mo(\X,Y,n)
= \sum_{\vc^{\sigma} \in \mathcal{T}}\ppm{\vc^{\sigma}}\sum_{i=1}^{R} \sum_{j=1}^{C}\frac{c^{\sigma}_{ij}}{n} \log \frac{ c^{\sigma}_{ij} n} {a_i b_j} 
\end{equation}
were $\ppm{\vc}$ is the probability of contingency table $\vc \in \mathcal{T}$.
This allows one to re-order terms to have a per cell contribution to $\mo$, rather than per contingency table $\vc \in \mathcal{T}$, i.e., 
\begin{equation}
\mo(\X,Y,n) =  \sum_{i=1}^{R} \sum_{j=1}^{C} \sum_{k=0}^{n}\ppm{c^{\sigma}_{ij}=k} \frac{k}{n} \log \frac{kn} {a_i b_j} \enspace .
\end{equation}
The individual empirical counts $c^{\sigma}_{ij}$ are then distributed hypergeometrically, i.e.,  
\[
\ppm{c^{\sigma}_{ij}=k}= \binom{b_i}{k}\binom{n-b_i}{a_j-k}/\binom{n}{a_j} \enspace . 
\]
These probabilities can be computed efficiently in an incremental manner by noting that they are non-zero only for $k$ between $\max(0,a_i+b_j-n)$ and $\min(a_i,b_j)$ and by using the hypergeometric recurrence formula.

\end{document}